\documentclass[9pt]{article} 
\usepackage{hyperref}
\usepackage{url}
\usepackage{smile}
\usepackage{graphicx} 
\usepackage{algorithm}
\usepackage{algorithmic}
\usepackage{epstopdf}
\usepackage[margin=1.0in]{geometry}
\usepackage[export]{adjustbox}
\usepackage{mathtools, cuted}
\IfFileExists{bbm.sty}{\usepackage{bbm}}{}
\usepackage{wrapfig}
\usepackage{subcaption}
\usepackage{caption}
\usepackage{enumitem}
\usepackage{tabularx}
\usepackage{tcolorbox}
\usepackage{placeins}
\usepackage{float}
\usepackage{longtable}


\numberwithin{equation}{section}

\usepackage[nocompress]{cite}

\hypersetup{
    colorlinks=true,
    linkcolor=blue,
    anchorcolor=blue,
    citecolor=blue,
}

\allowdisplaybreaks[4]

\let\oldfrac\frac
\renewcommand{\frac}[2]{\tfrac{#1}{#2}}

\newcommand{\Dcal}{\cD_{\mathrm{cal}}}
\newcommand{\Dwide}{\cD_{\mathrm{wide}}}
\newcommand{\Dfit}{\cD_{\mathrm{fit}}}
\newcommand{\Qnom}{Q_{\mathrm{nom}}}
\newcommand{\rhotr}{\rho_{\mathrm{tr}}}
\newcommand{\on}{\mathrm{on}}

\title{Robust Reinforcement Learning for Dynamic Treatment Regimes under Model and Measurement Uncertainty}
\author{
Chengze Xie\thanks{Department of Industrial and Systems Engineering,
Texas A\&M University. (E-mail: \url{changze@tamu.edu}).}
\and
Xuan Wang\thanks{Department of Mechanical Engineering,
The University of Texas at Austin. (E-mail: \url{xuan.wang@austin.utexas.edu}).}
\and
Ying Ding\thanks{School of Information,
The University of Texas at Austin. (E-mail: \url{ying.ding@ischool.utexas.edu}).}
\and
Hairong Wang\thanks{Department of Mechanical Engineering,
The University of Texas at Austin. (E-mail: \url{hairong.wang@austin.utexas.edu}).}
\and
Yan Li\thanks{Department of Industrial and Systems Engineering,
Texas A\&M University. Correspondence to: Yan Li (E-mail: \url{yan.li@tamu.edu}).}
}

\date{\vspace{-6ex}}
\begin{document}
\maketitle

\begin{abstract}
We develop a robust reinforcement learning (RL) framework for learning adaptive treatment policies that generalize across patients and remain reliable under errors in the measured biomarker and shifts in the disease dynamics. By modeling the patient as a hidden context, we address generalization and robustness in a unified manner through robust contextual Markov decision processes. Applying our framework to adaptive therapy for metastatic prostate cancer in a simulator calibrated to clinical data, we obtain generalizable and robust policies that substantially outperform the existing adaptive therapy method. In particular, we attain robustness with almost no loss of performance when the biomarker is accurate and the disease evolves as anticipated, and find that the resulting policy Pareto-dominates the existing method across a diverse array of scenarios. Our findings suggest that robust RL holds great promise as a principled methodology for learning reliable adaptive treatment policies.
\end{abstract}

\section{Introduction}\label{sec:intro}

Adapting treatment to the evolving condition of a patient has received considerable attention in medicine, with success reported in a broad range of applications \cite{murphy2003optimal,chakraborty2014dynamic,kosorok2019precision,komorowski2018artificial,brown2019six,yu2023reinforcement}. In oncology, adaptive therapy pursues this idea to delay the emergence of drug resistance \cite{gatenby2009adaptive}. Its rationale is that, under continuous treatment, drug-sensitive cells are eliminated and drug-resistant cells are left to grow without competition. By pausing treatment once the tumor has responded and resuming it once the tumor has regrown, clinicians can retain sensitive cells that continue to suppress the resistant ones \cite{gatenby2009adaptive,strobl2021turnover}. This approach was tested in a pilot clinical trial of the drug abiraterone for metastatic castration-resistant prostate cancer \cite{zhang2017integrating,zhang2022evolution}. In this clinical trial, clinicians paused abiraterone once prostate-specific antigen (PSA), a blood marker that tracks tumor burden, fell below half of its pretreatment value, and resumed it once PSA returned to the pretreatment value, a protocol we refer to as AT50. Compared with a contemporaneous cohort who received continuous treatment with abiraterone as the standard of care, the median time to progression more than doubled under AT50, while patients received no abiraterone during 46\% of their time in the clinical trial \cite{zhang2022evolution}.

Despite its remarkable success, AT50 is by design a simple heuristic in the sense that it applies the same pair of thresholds to every patient and responds only to the most recent PSA value. In addition, its benefit has so far been observed in only 17 patients, and, based on the models fitted to the same data, other schedules are expected to further improve outcomes \cite{zhang2022evolution}. Existing efforts toward more refined schedules include optimal control of competition models \cite{cunningham2018optimal}, deep reinforcement learning (RL) for controlling the emergence of drug resistance \cite{engelhardt2020dynamic}, and more recently RL policies trained within models calibrated to prostate cancer patients \cite{gallagher2024mathematical,lu2024deep,ma2026single}. Nevertheless, these policies are evaluated within the same models used for their training, and their performance beyond these models remains largely unexamined. Ideally, one would like to learn, in a data-driven manner, a dynamic treatment policy that is more sophisticated than AT50, attains improved performance, and generalizes to the broad population of patients beyond those enrolled in the clinical trial. More importantly, such a policy should remain reliable under the shifts between training and deployment that are common in clinical practice \cite{finlayson2021clinician}.

To be more specific, we identify three properties that a treatment policy learned from data should possess. To begin with, the policy should generalize in the sense that the performance it attains on the patients seen during training carries over to new patients drawn from the same population. From the perspective of contextual and hidden-parameter Markov decision processes (MDPs) \cite{hallak2015contextual,doshivelez2016hidden,kwon2021rl,steimle2021multi}, the natural goal without interacting with a new patient is a policy that is near-optimal on average over the population \cite{ye2023power}, and generalization across populations of environments has been studied extensively in deep RL \cite{kirk2023survey,tobin2017domain}. With a simulation model, we can pursue this goal by training on as many simulated patients as needed, far beyond the 17 patients of the pilot clinical trial, provided that the simulated population is broad enough to contain the patients of interest. In addition, the policy should be robust to errors in its observation. In particular, PSA serves only as an indirect measure of the underlying tumor burden, and radiographic progression can occur without any rise in PSA \cite{bryce2017radiographic}. Consequently, the quantity on which the policy acts may deviate from the quantity it aims to control, while RL policies are known to be sensitive to even small adversarial perturbations of their observations \cite{huang2017adversarial,zhang2020robust}. Most importantly, the policy should retain reliable performance under shifts in the patient dynamics. Indeed, the parameters of the model, including the strength of competition between sensitive and resistant cells on which the success of adaptive therapy hinges, are estimated from the sparse PSA records of a small cohort of patients, and are hence subject to considerable statistical error \cite{zhang2022evolution,browning2026adaptive}. Such shifts are particularly consequential in medical domains, where real patients are scarce and the dynamics of a patient may further change during treatment \cite{antonarakis2014arv7}.

To address these requirements, in this manuscript we propose to study adaptive therapy through the lens of robust contextual MDPs. A contextual MDP \cite{hallak2015contextual} describes a family of MDPs whose transition dynamics are indexed by a context, drawn from a population and hidden from the decision maker. Adaptive therapy fits naturally into this framework once we identify the context with the patient, namely the initial composition of the tumor together with its competition and growth parameters. Each patient then defines its own tumor dynamics, to be inferred by the clinician from the history of PSA measurements and treatment decisions, and AT50 is one particular policy acting on this history. By viewing the patient as a context, we treat generalization and robustness to shifts in the patient dynamics in a unified manner, since both concern the population from which the context is drawn. Specifically, generalization requires the policy to perform well on new contexts drawn from this population, which can be attained by training on sufficiently many sampled contexts, while robustness requires the policy to remain reliable when the context deviates from this population and thereby shifts the transition kernel that it determines.

To attain such robustness, we draw on robust MDPs, in which one seeks the best worst-case performance over an ambiguity set of transition kernels, and which have been studied extensively for standard MDPs \cite{iyengar2005robust,nilim2005robust,wiesemann2013robust,li2022first,morimoto2005robust,pinto2017robust,zhang2020robust}. We carry this viewpoint over to the contextual MDP and formulate robustness as a Markov game between the controller and the context, in which the context may deviate adversarially from the population from which it is drawn \cite{li2023rectangularity,li2026robust}. To account for errors in the observation, we further train the policy under random perturbations of the measured PSA. Within a model of abiraterone therapy fitted to clinical data \cite{zhang2022evolution}, which we refer to as the nominal model, the resulting robust policies substantially outperform AT50. In addition, we attain robustness to shifts in the patient dynamics at \emph{almost no price} in the nominal model, and the resulting policy \emph{nearly Pareto-dominates} the nominally trained policy, i.e., the policy trained in the nominal model with exact observations, under a broad range of realistic errors in the observation and shifts in the patient dynamics. Our contributions can be summarized as follows.

First, we formulate adaptive therapy as a contextual MDP in which the patient is modeled as a hidden context, and find that a policy learned within this formulation substantially outperforms AT50 and generalizes to patients fitted to clinical data. Within the competition model of \cite{zhang2022evolution}, the learned policy extends the mean progression-free time from 29.58 months under AT50 to 50.35 months, while treating in a comparable fraction of months. In addition, on the 32 patients whose initial tumor compositions are fitted to clinical data in \cite{zhang2022evolution}, policies trained on a broad synthetic population constructed without reference to these patients perform comparably to policies trained on a population calibrated to them. Hence, generalization to the patients of interest can be attained by training on sufficiently many simulated patients drawn from a broad population, without tailoring the population to these patients.

Second, we find that errors in PSA of realistic magnitude substantially reduce the benefit of policies that do not account for them. Specifically, under random errors of up to about 20\% of the true PSA value, the nominally trained policy loses a substantial part of its benefit, and AT50 performs only marginally better than continuous treatment. By training under such errors, we obtain a policy that retains the benefit of adaptive therapy in the sense that its performance under such errors nearly matches that without them.

Finally, we formulate robustness to shifts in the patient dynamics as a Markov game between the controller and the context over an ambiguity set of transition kernels \cite{li2026robust}, and find that the resulting kernel-robust policy nearly Pareto-dominates the nominally trained policy and AT50 across all the conditions we consider. Specifically, the kernel-robust policy (1) performs nearly as well as the nominally trained policy in the nominal model, (2) substantially outperforms every other type of policy under shifts in the dynamics, and (3) substantially improves on the nominally trained policy and AT50 even under errors in the observation. Our findings suggest that robustness to shifts in the patient dynamics can be attained at almost no price in the nominal model, and that the resulting policy remains robust against a broad spectrum of errors in the observation and shifts in the patient dynamics. In view of this, we believe that robust RL holds great promise as a principled methodology for learning adaptive therapy policies.

The rest of this manuscript is organized as follows. In Section~\ref{sec:formulation}, we introduce contextual MDPs and their robust counterparts, and formulate adaptive therapy within this framework. We present the numerical experiments in Section~\ref{sec:results}, and conclude the manuscript in Section~\ref{sec:conclusion}.

\section{Problem Formulation}\label{sec:formulation}

In this section, we discuss contextual MDPs and their robust counterparts, and formulate adaptive therapy for metastatic prostate cancer as a concrete instance. Specifically, in Section~\ref{sec:cmdp-general}, we introduce contextual MDPs and establish finite-sample guarantees on the number of sampled contexts needed to learn a generalizable policy. The concrete instance for adaptive therapy is then introduced in Section~\ref{sec:cmdp}, where we adopt the model of abiraterone therapy in \cite{zhang2022evolution}, which is fitted to clinical data from a pilot clinical trial and a contemporaneous cohort. In Section~\ref{sec:uncertainty}, we further introduce the robust counterparts of the contextual MDP against errors in the observation and shifts in the patient dynamics.

\subsection{Contextual Markov Decision Processes}\label{sec:cmdp-general}

We consider a finite-horizon contextual MDP \cite{hallak2015contextual,steimle2021multi} with state space $\cX$, action space $\cA$, observation space $\cY$, horizon $T$, and context space $\cC$. Each context $c\in\cC$ specifies an initial distribution $\mu_c$ over $\cX$, a transition kernel $P_c$, an observation map $O_c:\cX\to\cY$, and a reward function $r_c:\cX\times\cA\to\RR$, where $P_c(\cdot\given x,a)$ denotes the distribution of the next state when action $a$ is taken in state $x$. At the beginning of an episode, a context $c$ is drawn from a population distribution $\cD$ and remains fixed thereafter, after which the initial state is drawn as $x_0\sim\mu_c$, and the state at every period $t$ evolves as $x_{t+1}\sim P_c(\cdot\given x_t,a_t)$. The decision maker observes neither $c$ nor the state $x_t$, and receives at period $t$ the observation $y_t=O_c(x_t)$. Consequently, the action at period $t$ can only be based on the history $h_t=(y_0,a_0,\ldots,a_{t-1},y_t)$ of past observations and actions, i.e., a policy $\pi=(\pi_t)_{t\geq0}$ selects $a_t\sim\pi_t(\cdot\given h_t)$, and $\Pi$ denotes a given class of such history-dependent policies. For each context $c$, let $J(\pi;c):=\EE^\pi_c\sbr{\sum_{t=0}^{T-1}r_c(x_t,a_t)}$ denote the expected cumulative reward of $\pi$, where $\EE^\pi_c$ denotes the expectation over the trajectory generated by $\pi$ under $\mu_c$ and $P_c$. We consider
\begin{align}\label{eq:nominal}
\max_{\pi\in\Pi}\ J(\pi;\cD):=\EE_{c\sim\cD}\sbr{J(\pi;c)}.
\end{align}

In medical applications, the context collects the parameters that determine the dynamics of an individual patient, which vary across patients and are not observed by the decision maker. Since the context is hidden, the decision maker must infer it from the history, and an optimal policy for \eqref{eq:nominal} is in general history-dependent \cite{kaelbling1998planning,kwon2021rl}. In \eqref{eq:nominal}, we average the performance over the population of patients, which is the natural goal when a policy is deployed on a new patient without further interaction \cite{ye2023power}. In practice, the population distribution $\cD$ is unknown, and is accessible only through sampled contexts, such as the parameters of patients fitted to clinical data or of virtual patients generated in a simulator. Given contexts $c_1,\ldots,c_N$ drawn independently from $\cD$, we consider the sample average approximation (SAA) \cite{shapiro2021lectures} of \eqref{eq:nominal},
\begin{align}\label{eq:saa}
\max_{\pi\in\Pi}\ \hat J_N(\pi):=\oldfrac{1}{N}\sum_{i=1}^N J(\pi;c_i).
\end{align}
We proceed to show that an optimal solution of \eqref{eq:saa} generalizes to new contexts drawn from $\cD$ once sufficiently many contexts are sampled. To begin with, we establish such a guarantee for finite state, observation, and action spaces.

\begin{theorem}\label{thm:saa}
Suppose $\cX$, $\cY$, and $\cA$ are finite, and $0\leq r_c(x,a)\leq\bar r$ for all $(c,x,a)\in\cC\times\cX\times\cA$. Let $|\cH|=\sum_{t=0}^{T-1}|\cY|^{t+1}|\cA|^t$ denote the number of histories $h_t$ with $t\leq T-1$. Then for any class $\Pi$ of history-dependent policies, any $\hat\pi_N\in\argmax_{\pi\in\Pi}\hat J_N(\pi)$, and any $\delta\in(0,1)$, with probability at least $1-\delta$, we have
\begin{align*}
J(\hat\pi_N;\cD)\ \geq\ \sup_{\pi\in\Pi}J(\pi;\cD)-T\bar r\sqrt{\oldfrac{2(|\cA|-1)|\cH|\log\big(1+8T\sqrt N\big)+2\log(2/\delta)}{N}}-\oldfrac{T\bar r}{\sqrt N}.
\end{align*}
\end{theorem}

In view of Theorem~\ref{thm:saa}, the optimality gap of $\hat\pi_N$ on the population is $\cO\big(T\bar r\sqrt{|\cA||\cH|\log(TN)/N}\big)$ with high probability. Consequently, a policy trained on sufficiently many sampled contexts generalizes to new contexts drawn from $\cD$, provided that $\cD$ is broad enough to contain the contexts of interest. In Section~\ref{sec:res-gen}, we demonstrate such generalization for adaptive therapy by applying policies trained on a broad synthetic population to the patients fitted to clinical data in \cite{zhang2022evolution} without further training. The proof of Theorem~\ref{thm:saa} is deferred to Appendix~\ref{app:proof}.

On the other hand, the number $|\cH|$ of histories, and hence the number of sampled contexts required by Theorem~\ref{thm:saa}, grows exponentially in the horizon $T$. We proceed to show that the exponential dependence on $T$ can be avoided when the expected cumulative reward $J(\pi;c)$ is Lipschitz continuous with respect to the context $c$.

\begin{theorem}\label{thm:saa-wasserstein}
Suppose that $\cC\subset\RR^d$ is compact with diameter $\operatorname{diam}(\cC)$ under the Euclidean norm, $0\leq r_c(x,a)\leq\bar r$ for all $(c,x,a)\in\cC\times\cX\times\cA$, and there exists $L_T<\infty$ such that
\begin{align}\label{eq:context-lipschitz}
|J(\pi;c)-J(\pi;c')|\leq L_T\|c-c'\|_2\qquad\text{for all }\pi\in\Pi\text{ and }c,c'\in\cC.
\end{align}
Then there exists a constant $C_d$ depending only on $d$ such that for any $\hat\pi_N\in\argmax_{\pi\in\Pi}\hat J_N(\pi)$ and any $\delta\in(0,1)$, with probability at least $1-\delta$, we have
\begin{align*}
J(\hat\pi_N;\cD)\ \geq\ \sup_{\pi\in\Pi}J(\pi;\cD)-2L_T\operatorname{diam}(\cC)\left(\oldfrac{C_d\log(2N)}{N^{\min\{1/2,1/d\}}}+\sqrt{\oldfrac{\log(1/\delta)}{2N}}\right).
\end{align*}
\end{theorem}

In view of Theorem~\ref{thm:saa-wasserstein}, the exponential dependence on $T$ through $|\cH|$ in Theorem~\ref{thm:saa} is replaced by $L_T$, which grows only polynomially in $T$ when the dynamics are Lipschitz continuous with respect to the context. To be more specific, we establish in the following proposition the Lipschitz continuity \eqref{eq:context-lipschitz} for any class of policies, when the observation map does not depend on the context, and the initial distribution, the transition kernel, and the reward are Lipschitz continuous with respect to the context.

\begin{proposition}\label{prop:context-lipschitz}
Suppose $\cC\subset\RR^d$, $\cA$ is finite, the observation map $O_c=O$ is independent of $c$, $0\leq r_c(x,a)\leq\bar r$ for all $(c,x,a)\in\cC\times\cX\times\cA$, and there exists a constant $L\geq0$ such that
\begin{align*}
d_{\mathrm{TV}}(\mu_c,\mu_{c'})&\leq L\|c-c'\|_2,\\
d_{\mathrm{TV}}\big(P_c(\cdot\given x,a),P_{c'}(\cdot\given x,a)\big)&\leq L\|c-c'\|_2,\\
|r_c(x,a)-r_{c'}(x,a)|&\leq L\|c-c'\|_2
\end{align*}
for all $c,c'\in\cC$, $x\in\cX$, and $a\in\cA$, where $d_{\mathrm{TV}}(P,Q):=\sup_B|P(B)-Q(B)|$ denotes the total variation distance, and the supremum is taken over all measurable sets $B$. Then \eqref{eq:context-lipschitz} holds for any class $\Pi$ of history-dependent policies with
\begin{align*}
L_T=LT\left(1+\oldfrac{\bar r(T+1)}{2}\right).
\end{align*}
\end{proposition}

Given Proposition~\ref{prop:context-lipschitz}, we have $L_T=\cO\big(LT(1+\bar rT)\big)$, and hence Theorem~\ref{thm:saa-wasserstein} avoids the exponential dependence on $T$ in Theorem~\ref{thm:saa}. It is interesting to note that Theorem~\ref{thm:saa-wasserstein} does not require finite state, observation, and action spaces. On the other hand, the rate $\cO\big(N^{-\min\{1/2,1/d\}}\log N\big)$ deteriorates with the dimension $d$ once $d>2$. Nevertheless, such deterioration is mild for systems with a context of small dimension. For instance, the context of adaptive therapy in Section~\ref{sec:cmdp} is four-dimensional, which yields the rate $\cO(N^{-1/4}\log N)$. The proofs of Theorem~\ref{thm:saa-wasserstein} and Proposition~\ref{prop:context-lipschitz} are deferred to Appendix~\ref{app:proof}.

\subsection{Adaptive Therapy as a Contextual MDP}\label{sec:cmdp}

We adopt the model of abiraterone therapy in \cite{zhang2022evolution}, which describes the competition between drug-sensitive and drug-resistant cancer cells. Let $S(\tau)$ and $R(\tau)$ denote the numbers of drug-sensitive and drug-resistant cells at time $\tau$, measured in days, and let $u(\tau)\in\{0,1\}$ indicate whether abiraterone is administered. The dynamics are given by
\begin{align}\label{eq:lv}
\oldfrac{dS(\tau)}{d\tau}=r_S S(\tau)\left(1-\oldfrac{S(\tau)+R(\tau)}{\kappa\big(1-\lambda u(\tau)\big)}\right),
\qquad
\oldfrac{dR(\tau)}{d\tau}=r_R R(\tau)\left(1-\oldfrac{\alpha S(\tau)+R(\tau)}{\kappa}\right),
\end{align}
where $\kappa$ is the carrying capacity, i.e., the maximal number of cells supported by the available nutrients and space, and $r_S=\beta\gamma_S$ and $r_R=\beta\gamma_R$ are the growth rates of the sensitive and resistant cells per day, with base rates $\gamma_S,\gamma_R$ and a common factor $\beta$ scaling both rates.

In essence, \eqref{eq:lv} describes the different responses of the drug-sensitive and drug-resistant cells to abiraterone and the competition between them, which together underlie the rationale of adaptive therapy. Specifically, the sensitive cells depend on testosterone for survival and proliferation, while the resistant cells are independent of testosterone \cite{zhang2022evolution}. Since the synthesis of testosterone is inhibited by abiraterone, the sensitive cells utilize nutrients and space less efficiently under treatment, which is modeled in \eqref{eq:lv} by a reduction of their carrying capacity by the fraction $\lambda\in(0,1)$. In contrast, the resistant cells are not affected by abiraterone, and their growth is suppressed only through the competition from the sensitive cells, whose competitive effect on the resistant cells is measured by $\alpha$. Consequently, once the sensitive cells are depleted under continuous treatment, the resistant cells are released from competition and eventually dominate the tumor. In adaptive therapy, treatment is instead paused before such depletion occurs, which allows the sensitive cells to recover during the treatment break and to continue suppressing the resistant ones.

Following \cite{zhang2022evolution}, we set $(\lambda,\gamma_S,\gamma_R,\alpha,\beta)=(0.95,\,0.0156/6,\,0.0091/6,\,6,\,8)$, which are estimated by fitting \eqref{eq:lv} to the longitudinal PSA records of 32 patients, namely 17 patients of the pilot clinical trial and 15 patients of a contemporaneous cohort who received continuous treatment with abiraterone as the standard of care. We refer to $(\alpha,\beta)=(6,8)$, which are shared by all fitted patients, as the nominal values of the competition and growth parameters. We also take the initial cell numbers of each patient from the same fit (see Appendix~\ref{app:details}), and set $\kappa=10^4$ as in \cite{zhang2022evolution}, which only determines the unit of the cell numbers, since rescaling $S$, $R$, and $\kappa$ by a common factor leaves \eqref{eq:lv} unchanged. Since $\alpha$ and $\beta$ are estimated from the records of only 32 patients and are assumed to be shared by all of them, their values are subject to considerable statistical error, and may also differ across patients, as suggested in \cite{zhang2022evolution} by the poorer fits for three patients. Hence, in Sections~\ref{sec:uncertainty} and \ref{sec:results}, we allow both parameters to vary, and seek policies that remain reliable under shifts in them.

Since PSA is measured monthly in the pilot clinical trial \cite{zhang2022evolution}, the clinician decides at the beginning of every month of 28 days whether to administer abiraterone until the next visit, i.e., chooses $a_t\in\{0,1\}$ at visit $t\in\{0,1,\ldots,T-1\}$ and sets $u(\tau)=a_t$ for $\tau\in[28t,28(t+1))$. We take $T=65$, i.e., a horizon of approximately five years, which exceeds the median overall survival in both the adaptive therapy cohort and the standard-of-care cohort of \cite{zhang2022evolution}. At visit $t$, the state is given by the cell numbers $x_t=(S_t,R_t):=\big(S(28t),R(28t)\big)$. Since $u(\tau)$ is constant within each month, the state at the next visit is determined by $x_t$ and $a_t$ through the transition
\begin{align}\label{eq:transition}
x_{t+1}=F_{\alpha,\beta}(x_t,a_t):=z(28),\qquad\text{where}\quad\oldfrac{dz}{d\tau}=g_{\alpha,\beta}\big(z(\tau),a_t\big)\ \text{ for }\tau\in[0,28],\quad z(0)=x_t,
\end{align}
and $g_{\alpha,\beta}(x,u)$ denotes the right-hand side of \eqref{eq:lv} evaluated at $\big(S(\tau),R(\tau)\big)=x$ and $u(\tau)=u$, which depends on $\alpha$ through the competition term and on $\beta$ through the growth rates. We compute $F_{\alpha,\beta}$ by the classical fourth-order Runge-Kutta method with step size $1$ on the grid $\tau\in\{0,1,\ldots,28\}$, and truncate $S$ and $R$ to $[0,\kappa]$ after every step (see Appendix~\ref{app:details}).

We now formulate adaptive therapy as an instance of \eqref{eq:nominal}. For each patient, we take as context $c=(S_0,R_0,\alpha,\beta)$ the initial numbers of sensitive and resistant cells together with the competition and growth parameters. Given the context, the initial distribution $\mu_c$ is a point mass at $x_0=(S_0,R_0)$, and the transition kernel $P_c(\cdot\given x,a)$ is a point mass at $F_{\alpha,\beta}(x,a)$ defined in \eqref{eq:transition}. Consequently, the context determines the course of the disease under any sequence of treatment decisions. The clinician, however, observes neither the context nor the cell numbers, and only measures PSA. Indeed, sensitive and resistant cells are assumed in \cite{zhang2022evolution} to produce equal amounts of PSA, in which case PSA is given by the tumor burden $B_t=S_t+R_t$ at visit $t$ multiplied by the amount of PSA produced per cell. Since the amount of PSA produced per cell is unknown, the pretreatment PSA does not reveal the initial cell numbers $(S_0,R_0)$. We therefore take as the observation the logarithm of the ratio of PSA to its pretreatment value, which is given by $B_t/B_0$,
\begin{align}\label{eq:obs}
y_t=\log\left(B_t/B_0\right).
\end{align}

Under the protocol of the pilot clinical trial \cite{zhang2022evolution}, every patient receives abiraterone at enrollment, and treatment is reinitiated whenever PSA returns to or above its pretreatment value. We impose both requirements on every policy as a safety constraint, which we refer to as the forced-treatment rule. In terms of the observation \eqref{eq:obs}, we require $a_t=1$ at $t=0$ and at every visit with $y_t\geq0$, and impose no restriction on $a_t$ at any other visit. Accordingly, the policy class $\Pi$ in \eqref{eq:nominal} consists of the history-dependent policies satisfying this constraint.

We define progression as the first visit $t\geq1$ at which the tumor burden reaches $1.2B_0$, i.e., exceeds its pretreatment value by at least 20\%, which is the threshold adopted in \cite{strobl2021turnover,gallagher2025deriving,browning2026adaptive}. Since treatment can only be adjusted at visits, the tumor burden may reach $1.2B_0$ during a treatment break while the tumor is still partially sensitive to treatment. Such progression is referred to as premature progression in \cite{gallagher2025deriving}, where the threshold for pausing treatment is derived from the regrowth of the tumor between two consecutive visits to prevent it. Given that premature progression exposes the patient to a tumor burden that would be regarded as progression under treatment, we count progression irrespective of whether treatment is administered. In contrast, progression is declared only when the tumor grows under treatment in \cite{browning2026adaptive}, which concerns intermittent androgen suppression instead of abiraterone and adopts the model of \cite{strobl2021turnover} (see Appendix~\ref{app:pstd}). We define the progression-free time $G^\pi(c)\in\{0,1,\ldots,T\}$ as the number of months completed before progression. Correspondingly, we model progression as an absorbing state with zero reward, and set the reward at every visit $t$ before progression to
\begin{align*}
r_c(x_t,a_t)=\II\big\{B_{t+1}<1.2B_0\big\},\qquad\text{where}\quad(S_{t+1},R_{t+1})=F_{\alpha,\beta}(x_t,a_t).
\end{align*}
Hence the cumulative reward is given by $G^\pi(c)$, and $J(\pi;\cD)=\EE_{c\sim\cD}\sbr{\EE^\pi\sbr{G^\pi(c)}}$. It should be noted that AT50, implemented as in Section~\ref{sec:method}, is a feasible solution of \eqref{eq:nominal}, whose two thresholds are fixed in advance rather than obtained by solving \eqref{eq:nominal}. By applying RL to \eqref{eq:nominal}, we instead search over the entire class $\Pi$ for a policy that adapts the treatment to the whole history of PSA measurements and treatment decisions.

\subsection{Robust Contextual MDPs}\label{sec:uncertainty}

In clinical practice, a policy learned by solving the contextual MDP \eqref{eq:nominal} is subject to uncertainty beyond the variation across patients, which we model through the population distribution $\cD$ of contexts. In particular, the treatment decisions are made based on the measured PSA, which is affected by analytical and biological variation \cite{marques2022critical}, and serves only as an indirect measure of the underlying tumor burden even when measured exactly \cite{bryce2017radiographic}. The response of a patient to the treatment may also deviate from the nominal model in Section~\ref{sec:cmdp}. Indeed, the key parameters of the nominal model are estimated from a small cohort of patients and are hence subject to considerable statistical error, while the patient population and the disease dynamics may shift between learning and deployment \cite{finlayson2021clinician}, or even during treatment. For instance, tumors may develop new mechanisms of resistance to abiraterone during treatment \cite{antonarakis2014arv7}. Since the learned policy tailors its decisions to the PSA trajectories generated by the nominal model, its performance may deteriorate substantially in the presence of such uncertainty. In view of this, we seek policies that are robust against both errors in the observation and shifts in the patient dynamics.

To this end, we model the errors in the observation as random, and account for shifts in the patient dynamics through the lens of distributionally robust optimization (DRO) \cite{delage2010distributionally,bental2013robust,esfahani2018data,blanchet2019quantifying,gao2023distributionally,kuhn2025distributionally}, in which one seeks the best worst-case performance over an ambiguity set of distributions. Following the game-theoretic view of distributionally robust MDPs \cite{li2023rectangularity,li2026robust}, we formulate robustness to such shifts as a Markov game between the controller and an adversary acting on the transition kernel.

\subsubsection{Errors in the Observation}\label{sec:obs-robust}

To account for errors in the measured PSA, we model the observation at visit $t$ as
\begin{align}\label{eq:obs-pert}
y_t=\log\left(B_t/B_0\right)+\delta_t,\qquad |\delta_t|\leq\varepsilon,
\end{align}
where $\delta_t$ denotes the error at visit $t$, whose magnitude is bounded by $\varepsilon$, and the exact observation \eqref{eq:obs} corresponds to $\delta_t=0$. Both the policy and the forced-treatment rule act on the perturbed observation, and hence an error may also delay the restart of treatment required by the forced-treatment rule, which we refer to as the forced restart. Since the measured PSA is affected by analytical and within-subject biological variation \cite{marques2022critical}, we model the errors as random, and refer to them as random perturbations. Specifically, every $\delta_t$ is drawn independently from the uniform distribution on $[-\varepsilon,\varepsilon]$. We then learn policies that account for such errors by solving \eqref{eq:nominal} under the observation \eqref{eq:obs-pert}.

\subsubsection{Robustness to Shifts in the Patient Dynamics}\label{sec:kernel-robust}

To account for shifts in the patient dynamics, we allow the transition kernel of each patient to change in every month within an ambiguity set $\cP_c$ of transition kernels. Specifically, an adversary policy $\sigma$ selects a kernel from $\cP_c$ at visit $t$, possibly at random, based on the observations up to visit $t$, and we define $J(\pi,\sigma;\cD)$ as the expected cumulative reward of $\pi$ under the kernels selected by $\sigma$. We consider the robust counterpart of \eqref{eq:nominal} given by the zero-sum game
\begin{align}\label{eq:kernel-robust}
\max_{\pi\in\Pi}\ \min_{\sigma}\ J(\pi,\sigma;\cD).
\end{align}
Since the kernel is selected anew at every visit based on the realized history, \eqref{eq:kernel-robust} follows the dynamic formulation of robust MDPs \cite{li2026robust,li2026computational}, and hence accounts for shifts in the dynamics that occur during treatment.

To construct $\cP_c$ in adaptive therapy, we consider shifts in the competition and growth parameters $\alpha$ and $\beta$, whose estimates are subject to statistical error and may differ across patients (Section~\ref{sec:cmdp}). Following Remark~1 of \cite{li2026robust}, we draw $K$ pairs of multipliers $m^{(k)}=(m^{(k)}_\alpha,m^{(k)}_\beta)$ of $(\alpha,\beta)$, $k=1,\ldots,K$, once from a nominal distribution $\Qnom$ concentrated around one (see Appendix~\ref{app:details}), and keep them fixed. For a patient with context $c$, the $k$-th generating kernel $F^{(k)}_c$ is given by the transition \eqref{eq:transition} with parameters $(m^{(k)}_\alpha\alpha,m^{(k)}_\beta\beta)$, and $\cP_c$ consists of the mixtures of the $K$ generating kernels whose weights lie in
\begin{align}\label{eq:mixture}
\cW_\rho=\left\{(1-\rho)\,\tfrac{1}{K}\mathbf{1}+\rho\,q:\ q\in\Delta_K\right\},\qquad\rho\in[0,1],
\end{align}
where $\Delta_K$ denotes the probability simplex. The radius $\rho$ can be interpreted as the expected fraction of months in which the patient dynamics are selected by the adversary. Indeed, since the expected cumulative reward is linear in the mixture weights at every visit, it suffices for the adversary to place the weight $\rho$ on a single generating kernel at every visit. Equivalently, at every visit the kernel is drawn uniformly from the $K$ generating kernels with probability $1-\rho$, and is selected by the adversary with probability $\rho$.

Clearly, for $\rho=0$, problem \eqref{eq:kernel-robust} reduces to \eqref{eq:nominal} with the parameters redrawn uniformly from the generating kernels in every month, while for $\rho=1$ the adversary selects the kernel at every visit. In Section~\ref{sec:res-kernel}, we refer to the radius used in training as the training radius, denoted by $\rhotr$, and evaluate every policy at a range of radii $\rho$, which we refer to as evaluation radii.

\section{Numerical Experiments}\label{sec:results}

In this section, we demonstrate that the policies learned from the formulations in Section~\ref{sec:formulation} possess the three properties identified in Section~\ref{sec:intro}. After describing the experimental setup in Section~\ref{sec:method}, we show in Section~\ref{sec:res-main} that a policy learned from the contextual MDP \eqref{eq:nominal} substantially improves on AT50 without administering more drug. Such a policy, however, acts on PSA as if it were exact. In Section~\ref{sec:res-obs}, we find that errors in PSA of realistic magnitude substantially reduce the benefit of this policy, while by training under random perturbations we retain the benefit of adaptive therapy under such errors. We then turn to shifts in the patient dynamics, and show in Section~\ref{sec:res-kernel} that by solving the kernel-robust formulation \eqref{eq:kernel-robust}, we attain robustness to such shifts at almost no price in the nominal model, and in Section~\ref{sec:res-cross} that the resulting policy nearly Pareto-dominates the policy learned from \eqref{eq:nominal} and AT50 across all the sources of uncertainty we consider. Finally, in Section~\ref{sec:res-gen}, we attain on the 32 fitted patients a performance comparable to that obtained with a population calibrated to them, by training on a broad population constructed without reference to these patients.

\subsection{Experimental Setup}\label{sec:method}

We train all policies on a calibrated population $\Dcal$ of virtual patients, obtained by multiplying the initial cell numbers of the 32 fitted patients of \cite{zhang2022evolution} by independent lognormal factors, with the competition and growth parameters kept at their nominal values $(\alpha,\beta)=(6,8)$. For evaluation, we consider two test populations, namely the virtual test population of 1{,}000 patients drawn from $\Dcal$ independently of the training patients, and the 32 fitted patients, whose empirical distribution we denote by $\Dfit$. All methods are evaluated on the same test patients with common random numbers. The construction of $\Dcal$ and the evaluation protocol are given in Appendix~\ref{app:details}.

All learned policies are parameterized by a gated recurrent unit (GRU) network \cite{cho2014learning} acting on features of the history $h_t$, and trained by proximal policy optimization (PPO) \cite{schulman2017proximal} for 2{,}000 updates, each with 512 patients newly sampled from the training population. For each type of policy, we train three policies with different random seeds, and the remaining details are given in Appendix~\ref{app:details}.

We compare the learned policies with two clinical rules, namely continuous treatment, i.e., $a_t=1$ for all $t$, and AT50. In our model, we implement AT50 as the policy that pauses treatment once $y_t<\log0.5$, resumes it once $y_t\geq0$, and otherwise repeats its previous action. Both satisfy the forced-treatment rule.

Since we impose no explicit penalty in \eqref{eq:nominal} on patients who fare worse than under the standard of care, we monitor such patients separately. To this end, we compare the progression-free time $G^\pi(c)$ of each patient with the progression-free time $G^\on(c)$ of the same patient under continuous treatment. We say that $\pi$ harms the patient if $G^\pi(c)<G^\on(c)-3$, and that it benefits the patient if $G^\pi(c)>G^\on(c)+3$, where the margin of three months corresponds to the interval of 12 weeks between two consecutive imaging assessments in the pilot clinical trial \cite{zhang2022evolution}. We refer to the fractions of patients harmed and benefited as the harm rate and the benefit rate, and report them together with the mean progression-free time $\bar G$. For the learned policies, all values are averaged over the three seeds, with per-seed values in Appendix~\ref{app:results}.

\subsection{Substantial Improvement over AT50}\label{sec:res-main}

We begin with the contextual MDP \eqref{eq:nominal} under the exact observation \eqref{eq:obs}, and refer to the policy learned from \eqref{eq:nominal} as nominal RL. Table~\ref{tab:main} compares nominal RL with the two clinical rules on both test populations.

\begin{table}[!hb]
\centering
\caption{Nominal RL and the clinical rules in the nominal model, in terms of the mean progression-free time $\bar G$ (months) and its gain over continuous treatment, the harm and benefit rates, the fraction of patients without progression after $T=65$ months (PF at $T$), the fraction of treated months, and the mean number of treatment courses, i.e., maximal runs of consecutive treated months.}
\label{tab:main}
\small
\setlength{\tabcolsep}{4pt}
\begin{tabular}{lrrrrrrr}
\toprule
Method & $\bar G$ & Gain & Harm (\%) & Benefit (\%) & PF at $T$ (\%) & Treated (\%) & Courses\\
\midrule
\multicolumn{8}{l}{\emph{Virtual test population (1{,}000 patients)}}\\
Continuous treatment & 15.81 & -- & -- & -- & 2.7 & 100.0 & 1.0\\
AT50 & 29.58 & +13.77 & 28.4 & 54.0 & 32.5 & 28.7 & 6.2\\
Nominal RL & \textbf{50.35} & +34.54 & 0.2 & 91.9 & 61.5 & 30.3 & 12.9\\
\midrule
\multicolumn{8}{l}{\emph{Fitted patients (32 patients)}}\\
Continuous treatment & 14.66 & -- & -- & -- & 0.0 & 100.0 & 1.0\\
AT50 & 25.34 & +10.69 & 28.1 & 53.1 & 21.9 & 29.1 & 5.6\\
Nominal RL & \textbf{48.43} & +33.77 & 0.0 & 94.8 & 53.1 & 29.8 & 12.7\\
\bottomrule
\end{tabular}

\end{table}

\begin{figure}[t]
\centering
\includegraphics{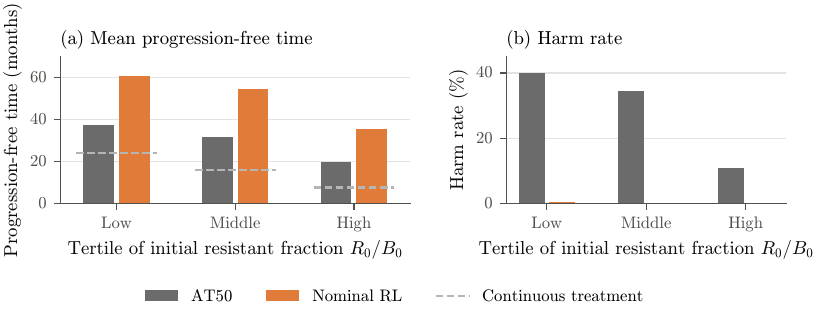}
\caption{Nominal RL, AT50, and continuous treatment on the virtual test population, split into tertiles of the initial resistant fraction $R_0/B_0$. (a) Mean progression-free time. (b) Harm rate.}
\label{fig:tertiles}
\end{figure}

Nominal RL substantially outperforms AT50. Specifically, it attains a mean progression-free time of 50.35 months on the virtual test population, compared with 29.58 months for AT50 and 15.81 months for continuous treatment. Its gain over continuous treatment is 34.54 months, more than twice the gain of 13.77 months attained by AT50. In addition, nominal RL harms 0.2\% of the patients and benefits 91.9\%, compared with 28.4\% and 54.0\% for AT50. On the 32 fitted patients, nominal RL attains 48.43 months, compared with 25.34 months for AT50, and harms none of the fitted patients.

A few remarks are in order. First, the improvement is attained without administering more drug. Nominal RL treats in 30.3\% of the months and AT50 in 28.7\%, while nominal RL divides its treated months into 12.9 courses on average, compared with 6.2 for AT50. Hence nominal RL alternates between treatment and breaks roughly twice as often as AT50, with correspondingly shorter courses and breaks. Second, the improvement holds across patients with different initial resistance (Figure~\ref{fig:tertiles}), a factor that correlates with the time to radiographic progression in both the adaptive therapy cohort and the standard-of-care cohort of \cite{zhang2022evolution}. When we split the virtual test population into tertiles of the initial resistant fraction $R_0/B_0$, nominal RL attains 60.81 months in the lowest tertile and 35.57 months in the highest, compared with 37.28 and 19.93 months for AT50 (Table~\ref{tab:tertiles}). AT50 mainly harms patients with few resistant cells, with a harm rate of 39.9\% in the lowest tertile compared with 10.8\% in the highest, while the harm rate of nominal RL is at most 0.4\% in every tertile. Third, the definition of progression in Section~\ref{sec:cmdp} is essential for an informative comparison. If we declare progression only when the tumor grows under treatment as in \cite{browning2026adaptive}, the growth of the tumor during a break is never penalized, and the rule that withholds treatment whenever permitted attains 53.79 months, the same value as all three seeds of nominal RL trained under this definition, compared with 52.60 months for AT50 (Appendix~\ref{app:pstd}). Consequently, such a definition does not allow us to assess the benefit of adapting the treatment to PSA.

In summary, by learning from data we obtain a policy that substantially improves on AT50 for every level of initial resistance without administering more drug. In addition, since the virtual test patients are drawn independently of the training patients, nominal RL generalizes to new patients from the training population, in line with Theorem~\ref{thm:saa}.

\subsection{Retaining the Benefit of Adaptive Therapy under Errors in PSA}\label{sec:res-obs}

Nominal RL acts on PSA as if it were exact. In clinical practice, however, PSA is subject to measurement errors. We therefore evaluate nominal RL under the random perturbations of Section~\ref{sec:obs-robust} with $\varepsilon=0.2$, which is comparable to the reference change value of about 19\% for total PSA, i.e., the smallest relative change between two consecutive measurements that can be distinguished from analytical and within-subject biological variation \cite{marques2022critical}. We also refer to the policy learned from \eqref{eq:nominal} under the same perturbations as random-perturbation RL, and compare it with nominal RL and the two clinical rules in Figure~\ref{fig:obs}.

\begin{figure}[t]
\centering
\includegraphics{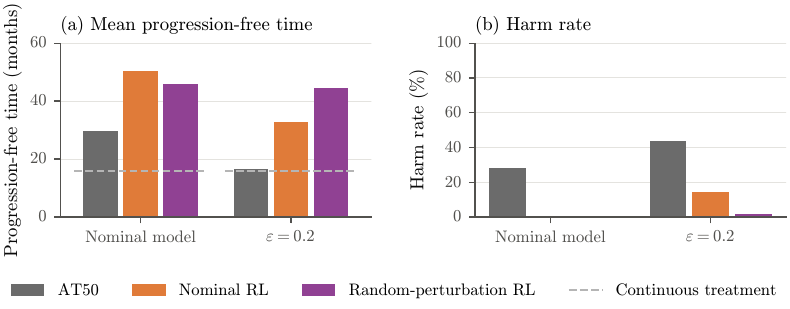}
\caption{Policies on the virtual test population in the nominal model and under random perturbations of PSA with $\varepsilon=0.2$. (a) Mean progression-free time. (b) Harm rate. See Table~\ref{tab:obs} of Appendix~\ref{app:results} for the values.}
\label{fig:obs}
\end{figure}

We observe that errors in PSA with $\varepsilon=0.2$ substantially reduce the benefit of the policies that do not account for them. Under random perturbations, nominal RL drops from 50.35 to 32.96 months, and its harm rate increases from 0.2\% to 14.4\%. Similarly, AT50 drops from 29.58 to 16.46 months, only marginally above the 15.81 months of continuous treatment, and its harm rate increases from 28.4\% to 43.8\%. In contrast, random-perturbation RL retains the benefit of adaptive therapy under such errors. Under random perturbations, it attains 44.49 months with a harm rate of 2.0\%, close to the 45.94 months it attains without perturbation.

A few remarks are in order. First, under perturbations with $\varepsilon=0.2$, we can no longer ensure through the forced-treatment rule that treatment is restarted before progression. Indeed, since $\varepsilon>\log1.2$, an error that lowers the measured PSA can postpone the forced restart until after progression. Since AT50 resumes treatment only through the condition $y_t\geq0$, such an error can postpone its restart in every treatment break. Second, we pay a price in the nominal model for accounting for errors in PSA, with 45.94 months for random-perturbation RL compared with 50.35 months for nominal RL.

In summary, policies that do not account for errors in PSA lose a substantial part of their benefit under errors of realistic magnitude, and by training under random perturbations we retain the benefit of adaptive therapy under such errors.

\subsection{Robustness to Shifts in the Patient Dynamics at Almost No Price}\label{sec:res-kernel}

We next turn to shifts in the patient dynamics, which arise when the competition and growth parameters deviate from their estimated values. We model such shifts by the ambiguity set $\cP_c$ of Section~\ref{sec:kernel-robust} with $K=32$ generating kernels drawn from $\Qnom$, and refer to the policy learned from \eqref{eq:kernel-robust} with training radius $\rhotr\in\{0.1,0.25,0.5\}$ as kernel-robust RL.

We evaluate all policies at the evaluation radii $\rho\in\{0,0.1,0.25,0.5,1\}$. At $\rho=0$, the kernel is drawn uniformly from the generating kernels in each month, and for $\rho>0$, each policy is evaluated against a best-response adversary, i.e., an adversary trained by PPO to minimize the reward of the policy while the policy is held fixed. In addition to the 32 generating kernels used in training, referred to as the training kernels, we evaluate the policies on 32 held-out kernels drawn independently from $\Qnom$. Figure~\ref{fig:kernel} shows the mean progression-free time and the harm rate in the nominal model and as functions of the evaluation radius, with the complete values in Table~\ref{tab:dyn-virtual}.

\begin{table}[!hb]
\centering
\caption{Mean progression-free time (months) under shifts in the patient dynamics on the virtual test population for policies trained on the calibrated population, with harm rates (\%) in parentheses.}
\label{tab:dyn-virtual}
\small
\setlength{\tabcolsep}{4pt}
\begin{tabular}{lcccccc}
\toprule
 & & \multicolumn{3}{c}{Kernel-robust RL} & & \\
\cmidrule(lr){3-5}
Condition & Nominal RL & $\rho_{\mathrm{tr}}=0.1$ & $\rho_{\mathrm{tr}}=0.25$ & $\rho_{\mathrm{tr}}=0.5$ & AT50 & Continuous treatment\\
\midrule
Nominal model & \textbf{50.35} (0.2) & 48.13 (0.1) & 48.77 (0.0) & 48.45 (0.0) & 29.58 (28.4) & 15.81\\
\midrule
\multicolumn{7}{l}{\emph{Training kernels}}\\
$\rho=0$ & 45.74 (8.9) & \textbf{53.45} (0.9) & 53.31 (1.0) & 53.07 (0.9) & 19.04 (44.4) & 16.11\\
$\rho=0.1$ & 44.20 (8.9) & \textbf{51.20} (0.9) & 51.06 (1.2) & 50.88 (1.0) & 19.13 (43.8) & 16.04\\
$\rho=0.25$ & 39.12 (12.2) & 46.49 (0.9) & \textbf{46.57} (1.3) & 46.50 (1.1) & 17.92 (45.9) & 15.68\\
$\rho=0.5$ & 30.98 (18.2) & 37.88 (0.8) & 38.44 (1.2) & \textbf{39.72} (0.8) & 15.51 (51.1) & 14.56\\
$\rho=1$ & 17.46 (29.1) & 23.38 (0.1) & 23.26 (0.6) & \textbf{25.08} (0.4) & 11.65 (65.1) & 12.72\\
\midrule
\multicolumn{7}{l}{\emph{Held-out kernels}}\\
$\rho=0$ & 47.26 (8.6) & \textbf{54.80} (1.0) & 54.21 (1.7) & 53.66 (1.7) & 20.23 (43.6) & 16.81\\
$\rho=0.1$ & 45.54 (9.4) & \textbf{52.80} (0.9) & 52.62 (1.2) & 52.36 (1.1) & 19.90 (44.4) & 16.71\\
$\rho=0.25$ & 39.30 (13.6) & 46.83 (0.6) & 46.87 (0.8) & \textbf{47.10} (0.7) & 18.80 (46.1) & 16.32\\
$\rho=0.5$ & 29.60 (19.9) & 36.67 (0.4) & 36.55 (0.5) & \textbf{37.85} (0.4) & 16.96 (49.8) & 15.27\\
$\rho=1$ & 17.44 (20.4) & 19.48 (0.1) & 19.36 (0.1) & \textbf{20.13} (0.1) & 12.11 (63.1) & 13.31\\
\bottomrule
\end{tabular}

\end{table}

\begin{figure}[t]
\centering
\includegraphics{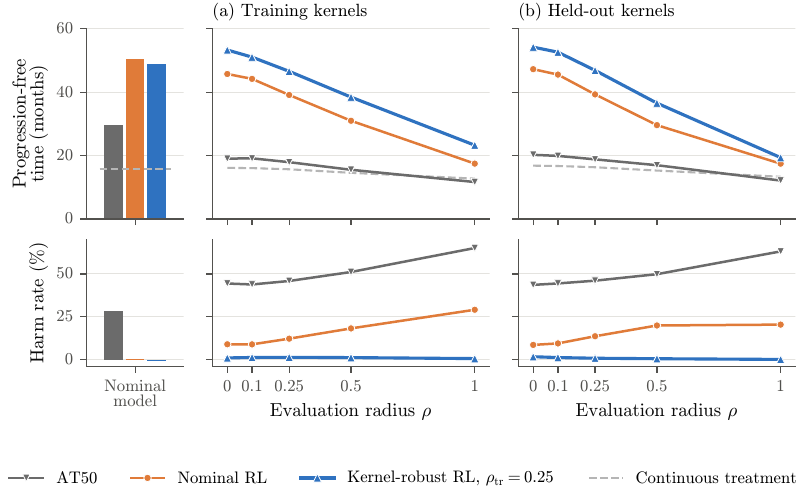}
\caption{Mean progression-free time (top) and harm rate (bottom) on the virtual test population. The left panels show the nominal model, and the middle and right panels show the performance as a function of the evaluation radius $\rho$, (a) on the training kernels and (b) on the held-out kernels. Kernel-robust RL is shown for $\rhotr=0.25$.}
\label{fig:kernel}
\end{figure}

Kernel-robust RL improves on nominal RL and AT50 at every evaluation radius on both sets of kernels. In the nominal model, it forgoes between 1.58 and 2.22 months compared with nominal RL, which is small compared with its improvement of between 7.33 and 7.71 months over nominal RL at $\rho=0$ on the training kernels. On the training kernels, kernel-robust RL evaluated at its training radius $\rho=\rhotr$ improves on nominal RL by between 7.00 and 8.74 months, and in particular attains 46.57 months at $\rho=0.25$, compared with 39.12 months for nominal RL. At $\rho=1$, which exceeds every training radius, kernel-robust RL with $\rhotr=0.5$ attains 25.08 months, compared with 17.46, 11.65, and 12.72 months for nominal RL, AT50, and continuous treatment. In addition, the harm rate of kernel-robust RL is at most 1.7\% at every radius on both sets of kernels, while on the training kernels, the harm rate of nominal RL grows from 8.9\% at $\rho=0$ to 29.1\% at $\rho=1$, and that of AT50 from 44.4\% to 65.1\%.

A few remarks are in order. First, nominal RL loses much of its advantage once the parameters vary, even in the absence of an adversary. Since the multipliers $m_\alpha$ of the generating kernels have mean larger than one, the resistant cells are on average suppressed more strongly at $\rho=0$ than in the nominal model. Indeed, at $\rho=0$, continuous treatment improves from 15.81 to 16.11 months, and kernel-robust RL improves from between 48.13 and 48.77 months to between 53.07 and 53.45 months. Nominal RL, however, drops from 50.35 to 45.74 months. Hence a policy trained for a single pair $(\alpha,\beta)$ is tuned to the pair used in training, and we benefit from exposing the policy to varying dynamics during training. Second, the best training radius increases with the evaluation radius. The policy with $\rhotr=0.1$ is best at $\rho\in\{0,0.1\}$, while the policy with $\rhotr=0.5$ is best at $\rho\in\{0.5,1\}$ on both sets of kernels. Hence the training radius can be chosen according to the magnitude of the shifts expected in deployment. Third, AT50 falls below continuous treatment at $\rho=1$ on both sets of kernels. Indeed, at $\rho=1$ the adversary can select the generating kernel with the weakest competition in every month, and adaptive therapy fails, or only delays failure, when the competitive effect of the sensitive cells on the resistant cells is small \cite{zhang2022evolution}.

In summary, by training against shifts in the patient dynamics, we retain the benefit of adaptive therapy under such shifts, including those generated by held-out kernels, while sacrificing almost no performance when the dynamics evolve as anticipated.

\subsection{Universal Benefits of Kernel-Robust RL}\label{sec:res-cross}

So far, we have evaluated random-perturbation RL and kernel-robust RL only under the source of uncertainty that each policy is trained against. In deployment, however, the source of uncertainty is not known in advance, and we would like a single policy that performs well under all of them. We therefore cross-evaluate the two policies, i.e., we evaluate random-perturbation RL under shifts in the dynamics and kernel-robust RL with training radius $\rhotr=0.25$ under random perturbations of PSA, and compare both policies with nominal RL and the two clinical rules in Figure~\ref{fig:cross}.

\begin{figure}[t]
\centering
\includegraphics{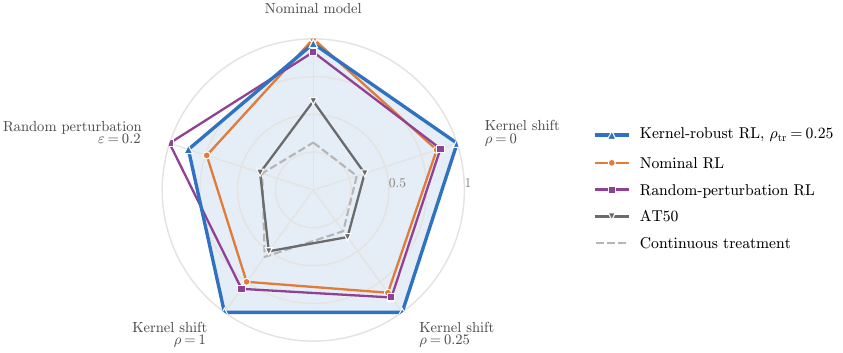}
\caption{Cross-evaluation on the virtual test population. Each axis shows the mean progression-free time divided by the highest value attained in the corresponding condition by the policies in this figure. Kernel shift denotes the evaluation radius $\rho$ on the training kernels, and random perturbation denotes errors in PSA with $\varepsilon=0.2$. Kernel-robust RL is trained with $\rhotr=0.25$. See Table~\ref{tab:cross} of Appendix~\ref{app:results} for the values.}
\label{fig:cross}
\end{figure}

We find that kernel-robust RL nearly Pareto-dominates nominal RL and AT50 across the five conditions of Figure~\ref{fig:cross}, namely the nominal model, the evaluation radii $\rho\in\{0,0.25,1\}$ on the training kernels, and random perturbations of PSA. Indeed, the polygon of kernel-robust RL encloses the polygon of AT50, and lies outside the polygon of nominal RL on every axis except the nominal model. In particular, kernel-robust RL forgoes 1.58 months compared with nominal RL in the nominal model, and attains the highest value under every shift in the dynamics, with 53.31, 46.57, and 23.26 months at $\rho\in\{0,0.25,1\}$, compared with 45.74, 39.12, and 17.46 months for nominal RL. Under random perturbations of PSA, which it never encounters during training, kernel-robust RL still improves on nominal RL, with 38.65 compared with 32.96 months. In addition, AT50 falls below kernel-robust RL, and indeed below all learned policies, in every condition.

On the other hand, by comparing kernel-robust RL with random-perturbation RL in Figure~\ref{fig:cross}, we observe that each of the two policies attains the highest value under the source of uncertainty it is trained against. Specifically, random-perturbation RL attains 44.49 months under random perturbations of PSA, compared with 38.65 months for kernel-robust RL, while kernel-robust RL attains higher values than random-perturbation RL under every shift in the dynamics.

In summary, kernel-robust RL nearly Pareto-dominates nominal RL and AT50 across the sources of uncertainty, in the sense that it nearly matches nominal RL in the nominal model, outperforms every other policy under shifts in the dynamics, and substantially improves on nominal RL and AT50 under errors in PSA. In view of this, we believe that kernel-robust training is a natural default for learning adaptive therapy policies, to be complemented by training under random perturbations when errors in PSA are of concern.

\subsection{Zero-Shot Generalization from a Broad Synthetic Population}\label{sec:res-gen}

By Theorem~\ref{thm:saa}, a policy trained on sufficiently many sampled patients generalizes to new patients drawn from the training population, provided that this population is broad enough to contain the patients of interest. So far, however, we have trained all policies on the calibrated population $\Dcal$, which is constructed from the 32 fitted patients. We now train each type of policy on a broad population $\Dwide$ constructed without reference to the fitted patients, with $\alpha\in[2,8]$, $\beta\in[6,10]$, and broad ranges of the initial tumor burden and the initial resistant fraction (see Appendix~\ref{app:details}). We then apply the resulting policies to the fitted patients without further training, which we refer to as zero-shot generalization, and evaluate them under shifts in the dynamics against best-response adversaries trained on the fitted patients. Figure~\ref{fig:gen} compares the two training populations, and Table~\ref{tab:gen} reports the policies trained on $\Dwide$ under errors in PSA and shifts in the dynamics.

\begin{figure}[t]
\centering
\includegraphics{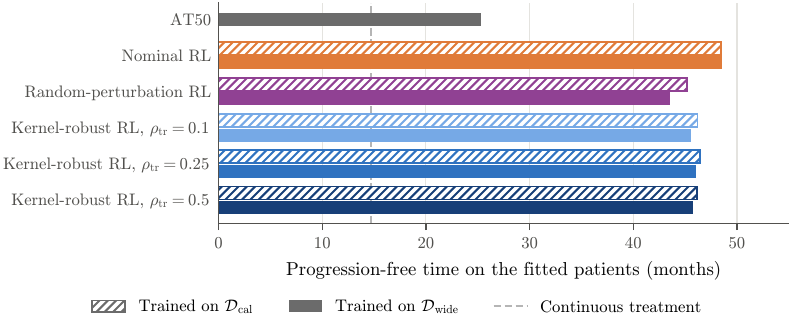}
\caption{Mean progression-free time on the 32 fitted patients in the nominal model for policies trained on the calibrated population and on the broad population.}
\label{fig:gen}
\end{figure}

\begin{table}[!htb]
\centering
\caption{Mean progression-free time (months) on the 32 fitted patients for policies trained on $\Dcal$ and on $\Dwide$ under the conditions of Table~\ref{tab:cross}.}
\label{tab:gen}
\footnotesize
\setlength{\tabcolsep}{4pt}
\begin{tabular}{lcccccc}
\toprule
 & Trained on $\Dcal$ & \multicolumn{5}{c}{Trained on $\Dwide$}\\
\cmidrule(lr){2-2}\cmidrule(lr){3-7}
Method & Nominal model & Nominal model & $\rho=0$ & $\rho=0.25$ & $\rho=1$ & $\varepsilon=0.2$\\
\midrule
Nominal RL & \textbf{48.43} & \textbf{48.59} & 46.14 & 33.99 & 13.64 & 26.85\\
Random-perturbation RL & 45.17 & 43.54 & 46.64 & 38.37 & \textbf{22.04} & \textbf{42.08}\\
Kernel-robust RL, $\rho_{\mathrm{tr}}=0.1$ & 46.20 & 45.58 & \textbf{53.34} & 44.32 & 20.03 & 32.77\\
Kernel-robust RL, $\rho_{\mathrm{tr}}=0.25$ & 46.46 & 46.03 & 52.84 & 43.31 & 20.79 & 31.90\\
Kernel-robust RL, $\rho_{\mathrm{tr}}=0.5$ & 46.11 & 45.77 & 53.31 & \textbf{44.59} & 21.18 & 32.94\\
\midrule
AT50 & 25.34 & 25.34 & 17.82 & 16.91 & 10.07 & 16.22\\
Continuous treatment & 14.66 & 14.66 & 14.97 & 14.37 & 10.90 & 14.66\\
\bottomrule
\end{tabular}

\end{table}

We observe almost no loss on the fitted patients from training on the broad population. Nominal RL trained on the broad population attains 48.59 months, compared with 48.43 months when trained on the calibrated population, and the two training populations differ by at most 1.63 months for every type of policy. In addition, we retain on the fitted patients the benefits of training under random perturbations and of kernel-robust training. Under random perturbations of PSA, random-perturbation RL attains 42.08 months, compared with 26.85 months for nominal RL. Under shifts in the dynamics on the training kernels, kernel-robust RL attains between 43.31 and 44.59 months at $\rho=0.25$ and between 20.03 and 21.18 months at $\rho=1$, compared with 33.99 and 13.64 months for nominal RL.

Kernel-robust RL also harms substantially fewer fitted patients than nominal RL and AT50 under the strongest shifts in the dynamics. Specifically, at $\rho=1$ on the training kernels, it harms between 8.4\% and 16.7\% of the fitted patients, compared with 47.3\% and 79.3\% for nominal RL and AT50, and on the held-out kernels it harms none of them, compared with 13.3\% and 63.6\% (Table~\ref{tab:dyn-real32}).

In summary, by training on a broad synthetic population constructed without the fitted patients, we attain on these patients a performance comparable to that obtained with a population calibrated to them, in line with Theorem~\ref{thm:saa}, and retain the benefits of both random-perturbation RL and kernel-robust RL.

\section{Concluding Remarks}\label{sec:conclusion}

In this manuscript, we study adaptive therapy through the lens of robust contextual MDPs, in which the patient is modeled as a hidden context drawn from a population. To account for the uncertainty beyond the variation across patients, we further train the policy under random perturbations of the measured biomarker and against an adversary acting on the transition kernel. Within a model of abiraterone therapy fitted to clinical data, we find that (1) a policy learned within this formulation substantially outperforms AT50 and generalizes to the fitted patients from a broad population constructed without them, (2) errors in PSA of realistic magnitude substantially reduce the benefit of policies that do not account for them, while by training under such errors we retain the benefit of adaptive therapy, and (3) we attain robustness to shifts in the patient dynamics at almost no price in the nominal model, with the resulting policy nearly Pareto-dominating nominal RL and AT50 across all the errors and shifts we consider. In view of this, we believe that robust RL offers a principled route toward adaptive treatment policies that remain reliable beyond the model in which they are learned.

\paragraph{Use of Large Language Models (LLMs).} We acknowledge the assistance of LLMs in prototyping the numerical experiments and improving the presentation of this manuscript.

\bibliographystyle{plain}
\bibliography{references}

\appendix

\section{Supplementary Proofs in Section~\ref{sec:formulation}}\label{app:proof}

To establish Theorem~\ref{thm:saa}, we define the distance between two history-dependent policies $\pi$ and $\pi'$ as
\begin{align*}
d_\Pi(\pi,\pi'):=\max_{0\leq t\leq T-1}\ \max_{h_t}\ \big\|\pi_t(\cdot\given h_t)-\pi'_t(\cdot\given h_t)\big\|_1,
\end{align*}
where the inner maximum is taken over all histories $h_t$ up to period $t$. We first show that the expected cumulative reward $J(\pi;c)$ is Lipschitz continuous in $\pi$ with respect to $d_\Pi$.

\begin{lemma}\label{lem:lipschitz}
Suppose $\cX$, $\cY$, and $\cA$ are finite, and $0\leq r_c(x,a)\leq\bar r$ for all $(c,x,a)\in\cC\times\cX\times\cA$. Then for any $c\in\cC$ and any pair of history-dependent policies $\pi,\pi'$, we have
\begin{align*}
\big|J(\pi;c)-J(\pi';c)\big|\leq\oldfrac{T^2\bar r}{2}\,d_\Pi(\pi,\pi').
\end{align*}
\end{lemma}

\begin{proof}
Let us fix $c\in\cC$, and denote by $z_t=(x_0,y_0,a_0,\ldots,a_{t-1},x_t,y_t)$ the full history up to period $t$, which consists of the history $h_t$ together with the states $x_0,\ldots,x_t$. To evaluate $\pi'$ along the trajectories generated by $\pi$, we set $V_T\equiv0$ and define recursively, for $t=T-1,\ldots,0$,
\begin{align*}
Q_t(z_t,a):=r_c(x_t,a)+\EE\sbr{V_{t+1}(z_{t+1})\given z_t,a_t=a},\qquad V_t(z_t):=\sum_{a\in\cA}\pi'_t(a\given h_t)\,Q_t(z_t,a),
\end{align*}
where the conditional expectation is taken over $x_{t+1}\sim P_c(\cdot\given x_t,a)$ with $y_{t+1}=O_c(x_{t+1})$, and is hence identical under $\pi$ and $\pi'$. By induction, we have
\begin{align*}
V_t(z_t)=\EE^{\pi'}_c\sbr{\sum_{s=t}^{T-1}r_c(x_s,a_s)\given z_t},\qquad0\leq Q_t(z_t,a)\leq(T-t)\bar r.
\end{align*}
In addition, since the distribution of $z_0$ does not depend on the policy, we have $J(\pi';c)=\EE^\pi_c\sbr{V_0(z_0)}$. Consequently,
\begin{align*}
J(\pi;c)-J(\pi';c)
&\overset{(a)}{=}\EE^\pi_c\sbr{\sum_{t=0}^{T-1}\big(r_c(x_t,a_t)+V_{t+1}(z_{t+1})-V_t(z_t)\big)}\\
&\overset{(b)}{=}\EE^\pi_c\sbr{\sum_{t=0}^{T-1}\sum_{a\in\cA}\big(\pi_t(a\given h_t)-\pi'_t(a\given h_t)\big)\Big(Q_t(z_t,a)-\oldfrac{(T-t)\bar r}{2}\Big)},
\end{align*}
where (a) follows from $V_T\equiv0$ and $J(\pi';c)=\EE^\pi_c\sbr{V_0(z_0)}$, and (b) follows from conditioning on $z_t$, the definitions of $Q_t$ and $V_t$, and $\sum_{a\in\cA}\big(\pi_t(a\given h_t)-\pi'_t(a\given h_t)\big)=0$. Hence,
\begin{align*}
\big|J(\pi;c)-J(\pi';c)\big|
&\overset{(c)}{\leq}\EE^\pi_c\sbr{\sum_{t=0}^{T-1}\big\|\pi_t(\cdot\given h_t)-\pi'_t(\cdot\given h_t)\big\|_1\max_{a\in\cA}\Big|Q_t(z_t,a)-\oldfrac{(T-t)\bar r}{2}\Big|}\\
&\overset{(d)}{\leq}\sum_{t=0}^{T-1}\oldfrac{(T-t)\bar r}{2}\,d_\Pi(\pi,\pi')=\oldfrac{T(T+1)\bar r}{4}\,d_\Pi(\pi,\pi')\leq\oldfrac{T^2\bar r}{2}\,d_\Pi(\pi,\pi'),
\end{align*}
where (c) follows from H\"older's inequality, and (d) follows from the definition of $d_\Pi$ and $0\leq Q_t(z_t,a)\leq(T-t)\bar r$. The proof is then completed.
\end{proof}

We next control the number of policies needed to cover the class of all history-dependent policies.

\begin{lemma}\label{lem:cover}
Suppose $\cY$ and $\cA$ are finite. Then for any $\nu>0$, there exist $K\leq(1+4/\nu)^{(|\cA|-1)|\cH|}$ history-dependent policies $\pi^1,\ldots,\pi^K$ such that
\begin{align*}
\min_{1\leq k\leq K}d_\Pi(\pi,\pi^k)\leq\nu\qquad\text{for any history-dependent policy }\pi.
\end{align*}
\end{lemma}

\begin{proof}
Let us denote by $\Delta_\cA$ the probability simplex over $\cA$, and take a maximal subset $S\subset\Delta_\cA$ satisfying $\|q-q'\|_1>\nu$ for all distinct $q,q'\in S$. By the maximality of $S$, we have
\begin{align}\label{eq:simplex-net}
\min_{q\in S}\|p-q\|_1\leq\nu\qquad\text{for any }p\in\Delta_\cA.
\end{align}
To control $|S|$, let us denote by $B(q,r)$ the $\ell_1$-ball of radius $r$ around $q$ within the affine hull of $\Delta_\cA$, which has dimension $|\cA|-1$. Since $\|p-q\|_1\leq2$ for all $p,q\in\Delta_\cA$, for any fixed $q_0\in S$, we have
\begin{align*}
B(q,\nu/2)\cap B(q',\nu/2)=\emptyset\quad\text{for all distinct }q,q'\in S,\qquad\bigcup_{q\in S}B(q,\nu/2)\subset B(q_0,2+\nu/2).
\end{align*}
By comparing the volumes in the above inclusion, we obtain
\begin{align*}
|S|\,(\nu/2)^{|\cA|-1}\leq(2+\nu/2)^{|\cA|-1},\quad\text{i.e.,}\quad|S|\leq(1+4/\nu)^{|\cA|-1}.
\end{align*}
Let $\{\pi^1,\ldots,\pi^K\}$ be the set of policies satisfying $\pi_t(\cdot\given h_t)\in S$ for any history $h_t$, in which case $K=|S|^{|\cH|}\leq(1+4/\nu)^{(|\cA|-1)|\cH|}$. For any history-dependent policy $\pi$, by choosing $\pi^k$ with $\pi^k_t(\cdot\given h_t)\in\argmin_{q\in S}\|\pi_t(\cdot\given h_t)-q\|_1$ for any $h_t$, we obtain $d_\Pi(\pi,\pi^k)\leq\nu$ from \eqref{eq:simplex-net}. The proof is then completed.
\end{proof}

\begin{proof}[Proof of Theorem~\ref{thm:saa}]
Let us set $\eta=T\bar r/(4\sqrt N)$ and $\nu=2\eta/(T^2\bar r)=1/(2T\sqrt N)$, and let $\pi^1,\ldots,\pi^K$ be the policies given by Lemma~\ref{lem:cover}, with $K\leq(1+8T\sqrt N)^{(|\cA|-1)|\cH|}$. Let us first fix $k$. Since $c_1,\ldots,c_N$ are drawn independently from $\cD$ and $0\leq r_c\leq\bar r$, the values $J(\pi^k;c_1),\ldots,J(\pi^k;c_N)$ are independent random variables taking values in $[0,T\bar r]$ with mean $J(\pi^k;\cD)$. By Hoeffding's inequality \cite{hoeffding1963probability}, we obtain
\begin{align*}
\PP\Big(\big|\hat J_N(\pi^k)-J(\pi^k;\cD)\big|\geq s\Big)\leq2\exp\left(-\oldfrac{2Ns^2}{T^2\bar r^2}\right)
\end{align*}
for any $s>0$. Let us take the union bound over $k=1,\ldots,K$ and choose
\begin{align*}
s=s_N:=T\bar r\sqrt{\oldfrac{\log(2K/\delta)}{2N}}.
\end{align*}
Then with probability at least $1-\delta$, we have
\begin{align}\label{eq:uniform-dev}
\max_{1\leq k\leq K}\big|\hat J_N(\pi^k)-J(\pi^k;\cD)\big|\leq s_N.
\end{align}
Let us now fix $\pi\in\Pi$ and some $k$ with $d_\Pi(\pi,\pi^k)\leq\nu$. Whenever \eqref{eq:uniform-dev} holds, we have
\begin{align}\label{eq:uniform-dev-all}
\big|\hat J_N(\pi)-J(\pi;\cD)\big|
&\leq\big|\hat J_N(\pi)-\hat J_N(\pi^k)\big|+\big|\hat J_N(\pi^k)-J(\pi^k;\cD)\big|+\big|J(\pi^k;\cD)-J(\pi;\cD)\big|\nonumber\\
&\overset{(a)}{\leq}s_N+2\eta,
\end{align}
where (a) follows from \eqref{eq:uniform-dev} and Lemma~\ref{lem:lipschitz} with $T^2\bar r\nu/2=\eta$. Since \eqref{eq:uniform-dev-all} holds for any $\pi\in\Pi$, we obtain
\begin{align*}
J(\hat\pi_N;\cD)\overset{(b)}{\geq}\hat J_N(\hat\pi_N)-s_N-2\eta\overset{(c)}{\geq}\hat J_N(\pi)-s_N-2\eta\overset{(d)}{\geq}J(\pi;\cD)-2s_N-4\eta,
\end{align*}
where (b) and (d) follow from \eqref{eq:uniform-dev-all}, and (c) follows from the definition of $\hat\pi_N$. In addition, we have $4\eta=T\bar r/\sqrt N$ and, since $K\leq(1+8T\sqrt N)^{(|\cA|-1)|\cH|}$,
\begin{align*}
2s_N=T\bar r\sqrt{\oldfrac{2\log(2K/\delta)}{N}}\leq T\bar r\sqrt{\oldfrac{2(|\cA|-1)|\cH|\log\big(1+8T\sqrt N\big)+2\log(2/\delta)}{N}}.
\end{align*}
The desired claim then follows by taking the supremum over $\pi\in\Pi$.
\end{proof}

\begin{proof}[Proof of Theorem~\ref{thm:saa-wasserstein}]
Let $\widehat\cD_N=N^{-1}\sum_{i=1}^N\delta_{c_i}$ denote the empirical distribution of $c_1,\ldots,c_N$, which satisfies $\hat J_N(\pi)=\EE_{c\sim\widehat\cD_N}\sbr{J(\pi;c)}$. To compare $J(\pi;\cD)$ with $\hat J_N(\pi)$, we use the $1$-Wasserstein distance \cite{villani2009optimal} between distributions $P$ and $Q$ on $\cC$, given by
\begin{align*}
W_1(P,Q):=\inf_{\gamma\in\Gamma(P,Q)}\int\|c-c'\|_2\,\gamma(dc,dc'),
\end{align*}
where $\Gamma(P,Q)$ denotes the set of couplings of $P$ and $Q$. Let us first fix $\pi\in\Pi$. For any $\gamma\in\Gamma(\cD,\widehat\cD_N)$, we have
\begin{align*}
\big|J(\pi;\cD)-\hat J_N(\pi)\big|=\left|\int\big(J(\pi;c)-J(\pi;c')\big)\,\gamma(dc,dc')\right|\leq L_T\int\|c-c'\|_2\,\gamma(dc,dc'),
\end{align*}
where the inequality follows from \eqref{eq:context-lipschitz}. Taking the infimum over $\gamma$ and then the supremum over $\pi\in\Pi$, we obtain
\begin{align}\label{eq:wasserstein-uniform}
\sup_{\pi\in\Pi}\big|J(\pi;\cD)-\hat J_N(\pi)\big|\leq L_TW_1(\cD,\widehat\cD_N).
\end{align}
Consequently, for any $\pi\in\Pi$,
\begin{align*}
J(\hat\pi_N;\cD)\overset{(a)}{\geq}\hat J_N(\hat\pi_N)-L_TW_1(\cD,\widehat\cD_N)\overset{(b)}{\geq}\hat J_N(\pi)-L_TW_1(\cD,\widehat\cD_N)\overset{(c)}{\geq}J(\pi;\cD)-2L_TW_1(\cD,\widehat\cD_N),
\end{align*}
where (a) and (c) follow from \eqref{eq:wasserstein-uniform}, and (b) follows from the definition of $\hat\pi_N$. Taking the supremum over $\pi\in\Pi$, we obtain
\begin{align}\label{eq:wasserstein-saa}
\sup_{\pi\in\Pi}J(\pi;\cD)-J(\hat\pi_N;\cD)\leq2L_TW_1(\cD,\widehat\cD_N).
\end{align}

It remains to control $W_1(\cD,\widehat\cD_N)$ with high probability. Since $W_1(\cD,\widehat\cD_N)=0$ whenever $\operatorname{diam}(\cC)=0$, and $W_1$ is invariant under a common translation of both distributions, without loss of generality, let us assume $\operatorname{diam}(\cC)>0$ and $0\in\cC$, in which case $\EE_{c\sim\cD}\sbr{\|c\|_2^3}\leq\operatorname{diam}(\cC)^3$. By Theorem~1 of \cite{fournier2015rate} with transport cost $\|c-c'\|_2$ and moment of order $3$, there exists a constant $C_d$ depending only on $d$ such that
\begin{align*}
\EE\sbr{W_1(\cD,\widehat\cD_N)}\leq \oldfrac{C_d\operatorname{diam}(\cC)\log(2N)}{N^{\min\{1/2,1/d\}}},
\end{align*}
where the logarithmic factor is needed only for $d=2$, and the additional term of order $N^{-2/3}$ in \cite{fournier2015rate} is absorbed into $C_d$. In addition, let $\widehat\cD_N^{(i)}$ denote the empirical distribution obtained by replacing $c_i$ with some $c_i'\in\cC$. By coupling every unchanged sample with itself and $c_i$ with $c_i'$, we have
\begin{align*}
\big|W_1(\cD,\widehat\cD_N)-W_1(\cD,\widehat\cD_N^{(i)})\big|\leq W_1(\widehat\cD_N,\widehat\cD_N^{(i)})\leq\oldfrac{\|c_i-c_i'\|_2}{N}\leq\oldfrac{\operatorname{diam}(\cC)}{N}.
\end{align*}
Hence it follows from McDiarmid's inequality \cite{mcdiarmid1989bounded} that
\begin{align*}
\PP\Big(W_1(\cD,\widehat\cD_N)-\EE\sbr{W_1(\cD,\widehat\cD_N)}\geq s\Big)\leq\exp\left(-\oldfrac{2Ns^2}{\operatorname{diam}(\cC)^2}\right)
\end{align*}
for any $s>0$. Let us choose $s=\operatorname{diam}(\cC)\sqrt{\log(1/\delta)/(2N)}$. Then with probability at least $1-\delta$, we have
\begin{align}\label{eq:wasserstein-concentration}
W_1(\cD,\widehat\cD_N)\leq\operatorname{diam}(\cC)\left(\oldfrac{C_d\log(2N)}{N^{\min\{1/2,1/d\}}}+\sqrt{\oldfrac{\log(1/\delta)}{2N}}\right).
\end{align}
The desired claim then follows by combining \eqref{eq:wasserstein-saa} and \eqref{eq:wasserstein-concentration}.
\end{proof}

\begin{proof}[Proof of Proposition~\ref{prop:context-lipschitz}]
Let us fix $\pi\in\Pi$ and $c,c'\in\cC$, and denote by $z_t=(x_0,y_0,a_0,\ldots,a_{t-1},x_t,y_t)$ the full history up to period $t$ as in the proof of Lemma~\ref{lem:lipschitz}. Since $O_c=O$, we have $y_t=O(x_t)$ under both $c$ and $c'$. For $\tilde c\in\{c,c'\}$, we set $V^{\tilde c}_T\equiv0$ and define recursively, for $t=T-1,\ldots,0$,
\begin{align*}
V^{\tilde c}_t(z_t):=\sum_{a\in\cA}\pi_t(a\given h_t)\Big(r_{\tilde c}(x_t,a)+\int V^{\tilde c}_{t+1}(z_t,a,x',O(x'))\,P_{\tilde c}(dx'\given x_t,a)\Big).
\end{align*}
By induction, we have $0\leq V^{\tilde c}_t(z_t)\leq(T-t)\bar r$ and $J(\pi;\tilde c)=\int V^{\tilde c}_0\,d\mu_{\tilde c}$. In addition, for any distributions $P$ and $Q$ and any measurable $f$ taking values in $[0,M]$, we have
\begin{align}\label{eq:tv-integral}
\Big|\int f\,dP-\int f\,dQ\Big|\leq M\,d_{\mathrm{TV}}(P,Q).
\end{align} Let us denote $\xi_t=\sup_{z_t}\big|V^c_t(z_t)-V^{c'}_t(z_t)\big|$. Then for any $z_t$,
\begin{align*}
\big|V^c_t(z_t)-V^{c'}_t(z_t)\big|
&\leq\sum_{a\in\cA}\pi_t(a\given h_t)\Big(\big|r_c(x_t,a)-r_{c'}(x_t,a)\big|+\Big|\int\big(V^c_{t+1}-V^{c'}_{t+1}\big)(z_t,a,x',O(x'))\,P_c(dx'\given x_t,a)\Big|\\
&\qquad+\Big|\int V^{c'}_{t+1}(z_t,a,x',O(x'))\,\big(P_c(dx'\given x_t,a)-P_{c'}(dx'\given x_t,a)\big)\Big|\Big)\\
&\overset{(a)}{\leq}L\|c-c'\|_2+\xi_{t+1}+(T-t-1)\bar rL\|c-c'\|_2,
\end{align*}
where (a) follows from the conditions on $r_c$ and $P_c$ in Proposition~\ref{prop:context-lipschitz}, the definition of $\xi_{t+1}$, and \eqref{eq:tv-integral} with $f=V^{c'}_{t+1}$ and $M=(T-t-1)\bar r$. Taking the supremum over $z_t$ and summing up the resulting relation over $t=0,\ldots,T-1$, we obtain from $\xi_T=0$ that
\begin{align*}
\xi_0\leq\sum_{t=0}^{T-1}\big(1+(T-t-1)\bar r\big)L\|c-c'\|_2=LT\left(1+\oldfrac{\bar r(T-1)}{2}\right)\|c-c'\|_2.
\end{align*}
Consequently,
\begin{align*}
\big|J(\pi;c)-J(\pi;c')\big|
&\leq\Big|\int\big(V^c_0-V^{c'}_0\big)\,d\mu_c\Big|+\Big|\int V^{c'}_0\,d\mu_c-\int V^{c'}_0\,d\mu_{c'}\Big|
\overset{(b)}{\leq}\xi_0+T\bar rL\|c-c'\|_2\\
&\leq LT\left(1+\oldfrac{\bar r(T+1)}{2}\right)\|c-c'\|_2,
\end{align*}
where (b) follows from the condition on $\mu_c$ in Proposition~\ref{prop:context-lipschitz} and \eqref{eq:tv-integral} with $f=V^{c'}_0$ and $M=T\bar r$. The proof is then completed.
\end{proof}

\section{Implementation Details}\label{app:details}

\subsection{Patient Model}

In the public code and data of \cite{zhang2022evolution}, the PSA records and the fitted initial cell numbers are available for 32 patients, namely 17 patients of the pilot clinical trial and 15 patients of the contemporaneous cohort who received continuous treatment with abiraterone as the standard of care. Since the records of 25 patients begin before the first dose of abiraterone, we follow the public code to reconstruct the daily treatment indicator of each patient from the recorded treatment history, integrate \eqref{eq:lv} with $(\alpha,\beta)=(6,8)$ from the fitted initial cell numbers along the reconstructed indicator, and take the cell numbers on the day of the first dose as the fitted initial state $(S_0,R_0)$. Throughout, we integrate \eqref{eq:lv} by the classical fourth-order Runge-Kutta method with step size $1$ in $\tau$, and truncate $S$ and $R$ to $[0,\kappa]$ after every step, as in the public code. For all 32 patients, we reproduce the trajectories computed by the public code up to a relative error below $10^{-13}$.

\subsection{Training Populations}

For $\Dcal$, we select one of the 32 fitted initial states uniformly at random, multiply $S_0$ and $R_0$ by independent lognormal factors with log-standard deviation $0.2$, and set $(\alpha,\beta)=(6,8)$. For $\Dwide$, we draw $\alpha\sim U[2,8]$, $\beta\sim U[6,10]$, $B_0\sim U[0.1\kappa,0.8\kappa]$, and the initial resistant fraction $R_0/B_0$ log-uniformly on $[0.001,0.5]$. The range of $\alpha$ extends from the value assumed in the earlier model of \cite{zhang2017integrating} to above the fitted value $6$, and the range of $\beta$ lies within 25\% of the fitted value $8$. We take the ranges of $B_0$ and $R_0/B_0$ from the priors of \cite{browning2026adaptive}, with the upper end of $B_0$ lowered from $\kappa$ to $0.8\kappa$ to keep the progression threshold $1.2B_0$ below the carrying capacity, and with $R_0/B_0$ drawn log-uniformly to adequately represent patients with few resistant cells. The virtual test population is drawn from $\Dcal$ with a fixed seed.

\subsection{Generating Kernels}

We take as $\Qnom$ the distribution of $m=(m_\alpha,m_\beta)$ with $\log m_\alpha\sim N(0,0.55^2)$ and $\log m_\beta\sim N(0,0.125^2)$ independent and truncated at two standard deviations. Under $\Qnom$, the competitive effect $m_\alpha\alpha$ at $\alpha=6$ ranges approximately over $[2,18]$, whose lower end corresponds to the value assumed in the earlier model of \cite{zhang2017integrating} and whose upper end remains within the range of competition coefficients reported among species in nature \cite{zhang2022evolution}. We draw the training kernels and the held-out kernels of Section~\ref{sec:res-kernel} independently from $\Qnom$ with different random seeds.

\subsection{Policy Network and Training}

The input of the policy network at visit $t$ consists of six features of $h_t$, namely $y_t$, the previous action, the elapsed time, the fraction of treated months, the time since the last change of action, and an indicator of forced treatment. We map these features through a linear layer with 64 units and $\tanh$ activation to a GRU with 64 hidden units, followed by linear heads for the action logits and the value. At visits with forced treatment, we exclude the action $a_t=0$ by masking its logit. The loss of PPO combines the clipped surrogate objective with clipping parameter $0.2$, the value loss with coefficient $0.5$, and the entropy bonus with coefficient $0.01$, where the advantages are computed by generalized advantage estimation with parameter $0.95$. In each update, we minimize the loss by the Adam optimizer with learning rate $3\times10^{-4}$ over four epochs with four minibatches, clip the gradient norm at $0.5$, and stop the update once the approximate Kullback-Leibler divergence between the new and old policies exceeds $0.03$. In training, we divide the reward $r_c$ of Section~\ref{sec:cmdp} by $T$. For every combination of policy type, training radius, and training population, we train three policies with the random seeds 11, 23, and 37, and report each policy after the final update, without selecting checkpoints on test data.

\subsection{Kernel Adversary}

We parameterize the kernel adversary by a network of the same architecture as the policy network, and train it by PPO with the same hyperparameters, using the negative of the reward of the controller. In solving \eqref{eq:kernel-robust}, we alternate between four adversary updates with the policy fixed and one policy update with the adversary fixed. The adversary selects the generating kernel of month $t$ before observing $a_t$, based on the observation $y_t$, the previous action $a_{t-1}$, the elapsed time $t/T$, and the one-hot encoding of the kernel it proposed in the previous month. We use its output head to parameterize a categorical distribution over the $K=32$ generating kernels. In accordance with \eqref{eq:mixture}, the patient evolves under the proposed kernel with probability $\rho$ and under a uniformly drawn generating kernel with probability $1-\rho$, and the adversary does not observe which of the two occurs. We initialize the output head of the adversary at zero, and hence the adversary proposes the kernels uniformly at random at initialization.

\subsection{Evaluation}

At evaluation, the learned policies select their most likely action at every visit. All methods are evaluated on the same test patients with the same random numbers, and on the fitted patients, each patient is evaluated under 50 independent replicates of the random perturbations. For every policy $\pi$ and evaluation radius $\rho>0$, we estimate the robust value $\min_\sigma J(\pi,\sigma;\cD)$ of $\pi$ by training a new best-response adversary against $\pi$, held fixed, for 2{,}000 PPO updates, on patients drawn from the calibrated population $\Dcal$ for the virtual test population and on the fitted patients for the fitted test population, and average the performance of the policy against this adversary, which samples its actions, over 32 independent repetitions. It should be noted that the best-response adversary is computed numerically, and the value attained against it hence provides an upper estimate of the exact robust value. AT50 and continuous treatment are evaluated against their own best-response adversaries. Under shifts in the dynamics, we compute $G^\on(c)$ under the same adversary, held fixed, acting on the continuous-treatment trajectory, with common random numbers.

\FloatBarrier
\section{Complete Numerical Results}\label{app:results}

To complement Section~\ref{sec:results}, we report the complete values underlying its findings. The comparison across levels of initial resistance in Section~\ref{sec:res-main} extends to both test populations, on which nominal RL improves on AT50 in every tertile of the initial resistant fraction, while the harm rate of AT50 is highest among patients with few resistant cells (Table~\ref{tab:tertiles}). The per-seed values of the learned policies under the conditions of Table~\ref{tab:cross} agree closely for nominal RL, random-perturbation RL, and kernel-robust RL in the nominal model and at $\rho=0$ (Table~\ref{tab:seeds}), and on the fitted patients, the three seeds of nominal RL attain 48.31, 48.66, and 48.31 months in the nominal model. The harm rates discussed in Section~\ref{sec:res-gen} for the policies trained on the broad population are given in Table~\ref{tab:dyn-real32}, together with their mean progression-free times under shifts in the patient dynamics on the fitted patients.

\begin{table}[!htb]
\centering
\caption{Mean progression-free time (months) on the virtual test population in the nominal model and under random perturbations of PSA with $\varepsilon=0.2$, with harm rates (\%) in parentheses.}
\label{tab:obs}
\small
\begin{tabular}{lcc}
\toprule
Method & Nominal model & $\varepsilon=0.2$\\
\midrule
Continuous treatment & 15.81 & 15.81\\
AT50 & 29.58 (28.4) & 16.46 (43.8)\\
\midrule
Nominal RL & \textbf{50.35} (0.2) & 32.96 (14.4)\\
Random-perturbation RL & 45.94 (0.1) & \textbf{44.49} (2.0)\\
\bottomrule
\end{tabular}

\end{table}

\begin{table}[!htb]
\centering
\caption{Cross-evaluation on the virtual test population (mean progression-free time in months). The columns $\rho\in\{0,0.25,1\}$ denote evaluation radii on the training kernels, and the column $\varepsilon=0.2$ denotes random perturbations of PSA.}
\label{tab:cross}
\small
\begin{tabular}{lccccc}
\toprule
 & & \multicolumn{3}{c}{Shifts in the dynamics} & Errors in PSA\\
\cmidrule(lr){3-5}\cmidrule(lr){6-6}
Method & Nominal model & $\rho=0$ & $\rho=0.25$ & $\rho=1$ & $\varepsilon=0.2$\\
\midrule
Continuous treatment & 15.81 & 16.11 & 15.68 & 12.72 & 15.81\\
AT50 & 29.58 & 19.04 & 17.92 & 11.65 & 16.46\\
\midrule
Nominal RL & \textbf{50.35} & 45.74 & 39.12 & 17.46 & 32.96\\
Random-perturbation RL & 45.94 & 47.15 & 40.91 & 18.80 & \textbf{44.49}\\
Kernel-robust RL, $\rho_{\mathrm{tr}}=0.25$ & 48.77 & \textbf{53.31} & \textbf{46.57} & \textbf{23.26} & 38.65\\
\bottomrule
\end{tabular}

\end{table}

{\small
\def\LTcap{\caption{Mean progression-free time (months) by tertile of the initial resistant fraction $R_0/B_0$ in the nominal model, with harm rates (\%) in parentheses. Tertiles are computed separately for each test population.}\label{tab:tertiles}}
\begin{longtable}{lccc}
\LTcap\\
\toprule
Initial resistant fraction & Nominal RL & AT50 & Continuous treatment\\
\midrule
\endfirsthead
\toprule
Initial resistant fraction & Nominal RL & AT50 & Continuous treatment\\
\midrule
\endhead
\bottomrule
\endfoot
\bottomrule
\endlastfoot
\multicolumn{4}{l}{\emph{Virtual test population}}\\
Low & 60.81 (0.4) & 37.28 (39.9) & 23.92\\
Middle & 54.70 (0.1) & 31.56 (34.5) & 15.89\\
High & 35.57 (0.0) & 19.93 (10.8) & 7.63\\
\midrule
\multicolumn{4}{l}{\emph{Fitted patients}}\\
Low & 62.30 (0.0) & 31.40 (50.0) & 22.50\\
Middle & 54.55 (0.0) & 28.91 (27.3) & 14.27\\
High & 29.70 (0.0) & 16.27 (9.1) & 7.91\\
\end{longtable}
}

{\small
\setlength{\tabcolsep}{4.5pt}
\def\LTcap{\caption{Per-seed mean progression-free time (months) on the virtual test population for the conditions of Table~\ref{tab:cross}.}\label{tab:seeds}}
\begin{longtable}{llccccc}
\LTcap\\
\toprule
Method & Seed & Nominal model & $\rho=0$ & $\rho=0.25$ & $\rho=1$ & $\varepsilon=0.2$\\
\midrule
\endfirsthead
\toprule
Method & Seed & Nominal model & $\rho=0$ & $\rho=0.25$ & $\rho=1$ & $\varepsilon=0.2$\\
\midrule
\endhead
\bottomrule
\endfoot
\bottomrule
\endlastfoot
Nominal RL & 11 & 50.33 & 46.00 & 39.91 & 14.87 & 34.66\\
 & 23 & 50.32 & 46.25 & 39.56 & 15.97 & 31.85\\
 & 37 & 50.40 & 44.96 & 37.89 & 21.53 & 32.38\\
\midrule
Random-perturbation RL & 11 & 45.76 & 47.58 & 41.47 & 17.11 & 44.50\\
 & 23 & 46.12 & 47.35 & 41.51 & 18.99 & 44.46\\
 & 37 & 45.95 & 46.52 & 39.75 & 20.29 & 44.51\\
\midrule
Kernel-robust RL, $\rho_{\mathrm{tr}}=0.1$ & 11 & 47.95 & 53.58 & 46.15 & 20.43 & 38.86\\
 & 23 & 48.57 & 53.45 & 46.73 & 25.28 & 36.31\\
 & 37 & 47.88 & 53.30 & 46.58 & 24.42 & 39.72\\
\midrule
Kernel-robust RL, $\rho_{\mathrm{tr}}=0.25$ & 11 & 48.71 & 53.48 & 46.61 & 23.12 & 38.15\\
 & 23 & 48.34 & 53.54 & 46.57 & 25.19 & 38.11\\
 & 37 & 49.25 & 52.90 & 46.54 & 21.48 & 39.67\\
\midrule
Kernel-robust RL, $\rho_{\mathrm{tr}}=0.5$ & 11 & 48.66 & 53.02 & 46.11 & 24.66 & 38.27\\
 & 23 & 48.13 & 53.19 & 46.93 & 25.77 & 36.10\\
 & 37 & 48.56 & 52.98 & 46.48 & 24.79 & 36.38\\
\end{longtable}
}

{\small
\setlength{\tabcolsep}{4pt}
\def\LTcap{\caption{Shifts in the patient dynamics on the 32 fitted patients, policies trained on the broad population. Notation as in Table~\ref{tab:dyn-virtual}.}\label{tab:dyn-real32}}
\begin{longtable}{lcccccc}
\LTcap\\
\toprule
 & & \multicolumn{3}{c}{Kernel-robust RL} & & \\
\cmidrule(lr){3-5}
Condition & Nominal RL & $\rho_{\mathrm{tr}}=0.1$ & $\rho_{\mathrm{tr}}=0.25$ & $\rho_{\mathrm{tr}}=0.5$ & AT50 & Continuous treatment\\
\midrule
\endfirsthead
\toprule
 & & \multicolumn{3}{c}{Kernel-robust RL} & & \\
\cmidrule(lr){3-5}
Condition & Nominal RL & $\rho_{\mathrm{tr}}=0.1$ & $\rho_{\mathrm{tr}}=0.25$ & $\rho_{\mathrm{tr}}=0.5$ & AT50 & Continuous treatment\\
\midrule
\endhead
\bottomrule
\endfoot
\bottomrule
\endlastfoot
Nominal model & 48.59 (0.0) & 45.58 (0.0) & 46.03 (0.0) & 45.77 (0.0) & 25.34 (28.1) & 14.66\\
\midrule
\multicolumn{7}{l}{\emph{Training kernels}}\\
$\rho=0$ & 46.14 (7.1) & 53.34 (0.7) & 52.84 (0.9) & 53.31 (0.9) & 17.82 (45.4) & 14.97\\
$\rho=0.1$ & 42.08 (10.3) & 50.15 (0.9) & 50.03 (1.1) & 50.19 (1.4) & 17.72 (46.1) & 14.83\\
$\rho=0.25$ & 33.99 (15.3) & 44.32 (1.3) & 43.31 (2.9) & 44.59 (2.0) & 16.91 (48.3) & 14.37\\
$\rho=0.5$ & 24.40 (24.2) & 34.51 (5.8) & 33.13 (8.4) & 34.81 (6.2) & 14.18 (53.9) & 12.85\\
$\rho=1$ & 13.64 (47.3) & 20.03 (16.7) & 20.79 (12.3) & 21.18 (8.4) & 10.07 (79.3) & 10.90\\
\midrule
\multicolumn{7}{l}{\emph{Held-out kernels}}\\
$\rho=0$ & 47.29 (8.5) & 56.35 (0.9) & 55.82 (1.1) & 55.95 (1.3) & 18.38 (47.9) & 15.65\\
$\rho=0.1$ & 43.77 (10.7) & 52.59 (1.0) & 51.99 (1.3) & 52.25 (1.5) & 18.52 (46.9) & 15.46\\
$\rho=0.25$ & 34.36 (16.7) & 45.20 (1.3) & 43.99 (2.4) & 45.22 (1.5) & 17.24 (50.5) & 14.91\\
$\rho=0.5$ & 24.90 (25.8) & 33.93 (4.1) & 33.11 (4.7) & 34.60 (3.2) & 14.92 (56.3) & 13.76\\
$\rho=1$ & 18.79 (13.3) & 20.86 (0.0) & 20.68 (0.0) & 20.74 (0.0) & 10.89 (63.6) & 11.64\\
\end{longtable}
}

\FloatBarrier
\section{Alternative Definition of Progression}\label{app:pstd}

To justify the definition of progression in Section~\ref{sec:cmdp}, we compare it with an alternative definition, under which progression is declared only when the tumor grows under treatment as in \cite{browning2026adaptive}. Specifically, we declare progression at the first visit $t$ with $a_{t-1}=1$, $B_t\geq1.2B_0$, and $B_t>B_{t-1}$, i.e., when the tumor exceeds the threshold during treatment and is still growing. Growth during a treatment break is then never penalized, and the forced restart at $y_t\geq0$ occurs below the threshold $1.2B_0$. We report the results in the nominal model under this definition in Table~\ref{tab:pstd}, together with the forced-only rule, which withholds treatment whenever permitted, i.e., sets $a_t=1$ if and only if $t=0$ or $y_t\geq0$. All three seeds of nominal RL, trained under this definition, attain 53.79 months on the virtual test population and 53.75 months on the fitted patients, which coincide with the values of the forced-only rule, while AT50 attains 52.60 and 53.06 months. Consequently, under the alternative definition, the benefit of adapting the treatment to PSA beyond the forced-treatment rule cannot be assessed.

{\small
\def\LTcap{\caption{Mean progression-free time (months) under the alternative definition of progression in the nominal model. Forced only denotes the rule that treats only when forced.}\label{tab:pstd}}
\begin{longtable}{lccccc}
\LTcap\\
\toprule
Population & Continuous & AT50 & Forced only & Nominal RL & Seeds 11 / 23 / 37\\
\midrule
\endfirsthead
\toprule
Population & Continuous & AT50 & Forced only & Nominal RL & Seeds 11 / 23 / 37\\
\midrule
\endhead
\bottomrule
\endfoot
\bottomrule
\endlastfoot
Virtual & 15.81 & 52.60 & 53.79 & 53.79 & 53.79 / 53.79 / 53.79\\
Fitted patients & 14.66 & 53.06 & 53.75 & 53.75 & 53.75 / 53.75 / 53.75\\
\end{longtable}
}

\end{document}